\documentclass[11pt,a4paper]{article}

\usepackage[margin=25mm]{geometry}
\usepackage{amsmath,amssymb,amsthm}
\usepackage{enumitem}
\usepackage{hyperref}
\usepackage[nameinlink,noabbrev]{cleveref}
\usepackage{microtype}
\usepackage{authblk}

\hypersetup{
  colorlinks=true,
  linkcolor=blue,
  citecolor=blue,
  urlcolor=blue,
  pdftitle={Exact ReLU realization of affine one-dimensional refinement iterates via residual memory and offset frames},
  pdfauthor={Boldsaikhan Bolorkhuu and Tsogtgerel Gantumur},
  pdfkeywords={ReLU networks, affine refinement, residual memory, offset frames, piecewise linear functions}
}

\usepackage{aliascnt}
\usepackage{cleveref}

\newtheorem{theorem}{Theorem}[section]

\newaliascnt{proposition}{theorem}
\newtheorem{proposition}[proposition]{Proposition}
\aliascntresetthe{proposition}

\newaliascnt{lemma}{theorem}
\newtheorem{lemma}[lemma]{Lemma}
\aliascntresetthe{lemma}

\newaliascnt{corollary}{theorem}
\newtheorem{corollary}[corollary]{Corollary}
\aliascntresetthe{corollary}

\theoremstyle{definition}

\newaliascnt{definition}{theorem}
\newtheorem{definition}[definition]{Definition}
\aliascntresetthe{definition}

\newaliascnt{remark}{theorem}
\newtheorem{remark}[remark]{Remark}
\aliascntresetthe{remark}

\newcommand{\R}{\mathbb{R}}
\newcommand{\Z}{\mathbb{Z}}
\newcommand{\supp}{\operatorname{supp}}
\newcommand{\vect}{\operatorname{Vec}}
\newcommand{\ReLU}{\operatorname{ReLU}}
\newcommand{\Ups}{\Upsilon}
\newcommand{\transpose}{\mathsf{T}}
\newcommand{\mm}{j_0}
\newcommand{\mar}{\varrho}

\newcommand{\doi}[1]{\href{https://doi.org/#1}{DOI: #1}}

\title{Exact ReLU realization of affine one-dimensional refinement iterates via residual memory and offset frames}
\author[2]{Boldsaikhan Bolorkhuu}
\author[1,2,3]{Tsogtgerel Gantumur}

\affil[1]{McGill University}
\affil[2]{National University of Mongolia}
\affil[3]{Institute of Mathematics and Digital Technology, Mongolian Academy of Sciences}

\date{July 20, 2026}

\begin{document}
\maketitle

\begin{abstract}
We study vector-valued affine refinement operators of the form
\[
(W\gamma)(t)=\sum_{j\in\mathbb Z}A_j\gamma(Mt-j)+B(t),
\]
with finitely supported matrix mask
and compactly supported continuous piecewise linear input and forcing data.  Building on the
homogeneous realization theorem for \(B\equiv0\), we prove that, for \(M\ge3\), every finite affine
iterate \(W^n\gamma\) admits an exact fixed-width ReLU realization whose depth is \(O(n)\).

The main new ingredient is a residual memory controller.  It replaces the noninvertible residual
dynamics by an injective skew-product and permits exact backward replay of the residual states
required by a Horner-type evaluation of the affine forcing sum. 
Offset frames align the forcing atoms away from residual seams, allowing complementary loop
readouts to recover their values exactly.  The remaining branch-selection ambiguity occurs only
where the accumulated affine state has already vanished.

For \(M\ge3\), the result applies to arbitrary compactly supported continuous piecewise linear
forcing terms.  For \(M=2\), the same construction applies to ordinary-frame seam-separated
forcing.  We also prove a stage-dependent extension for forcing terms in a fixed finite-dimensional
continuous piecewise linear span and record the resulting linear-depth upgrade for open-curve,
finite-state, and Hilbert- and Morton-type recursive constructions.
\end{abstract}



%\tableofcontents

\section{Introduction}\label{sec:introduction}

\subsection{Background and motivation}

A central question in neural-network approximation theory \cite{approx,nnapprox} is to explain why
deep networks can represent highly oscillatory, recursive, or self-similar functions with relatively
small width and depth.  A basic result in this direction is the scalar binary theorem for refinement
operators \cite{source}: if the input function is compactly supported and continuous piecewise
linear (CPwL), then its finite refinement iterates admit exact fixed-width ReLU realizations whose
depth grows only linearly with the number of refinement steps.

The scalar binary construction of \cite{source} naturally suggests extensions to 
vector-valued, \(M\)-ary refinement.  A homogeneous theory in this setting was developed in \cite{loop} using
an exact loop controller for the residual dynamics.  That work proves fixed-width,
depth-\(O(n)\) realizations for homogeneous refinement iterates and also treats affine refinement,
but with depth \(O(n^2)\).  The corresponding linear-depth realization of genuinely affine forcing
sums therefore remained open for the \(M\)-ary vector-valued systems arising in recursive curve
generation.

The main motivation is geometric.  Finite approximants to recursive constructions, including
Hilbert-type curves and Morton-type traversals, can be generated by vector-valued refinement rules
with affine connector terms.  This parametrized-generation viewpoint complements \cite{fractals},
where neural networks are used to approximate indicators or classifiers of fractal sets.  In the
present paper, we take the homogeneous theorem of \cite{loop} as a black box and address the missing
affine problem: the exact fixed-width realization of the forcing sum with depth \(O(n)\).

The affine refinement operator considered here has the form
\begin{equation}
\label{eq:intro-W}
(W\gamma)(t)
=
\sum_{j\in\mathbb Z}A_j\gamma(Mt-j)+B(t),
\end{equation}
where \(M\ge2\), the matrices \(A_j\in\R^{p\times p}\) are fixed and vanish for all but finitely
many \(j\), and \(B:\R\to\R^p\) is compactly supported and CPwL.  Writing the homogeneous part as
\[
(V\gamma)(t)
:=
\sum_{j\in\mathbb Z}A_j\gamma(Mt-j),
\]
one obtains the affine iterate identity
\begin{equation}
\label{eq:intro-affine-decomposition}
W^n\gamma
=
V^n\gamma+\sum_{r=0}^{n-1}V^rB .
\end{equation}
The first term is covered by the homogeneous theorem of \cite{loop}.  The new question is
whether the affine forcing sum can also be realized exactly with fixed width and depth \(O(n)\).

This question is subtler than the homogeneous one.  In the homogeneous case, the cascade can be
organized as a forward adjoint iteration along the residual orbit, while the scalar terminal factor
makes certain selector ambiguities harmless.  For affine forcing, the terms enter at many different
stages, and the efficient organization is instead a backward Horner-type evaluation of the affine
sum.  This creates a chronology problem: after the residual controller has advanced to depth \(n\),
the network must recover the earlier residual states in the reverse order required by the affine
recursion.

The main new device of this paper is therefore a residual memory controller.  It augments the loop
state by a fixed-dimensional memory coordinate so that the forward controller is injective on the
relevant state space and has a CPwL inverse on its image.  After running forward to depth \(n\), the
network can consequently replay the residual states backward exactly and feed them into the affine
recursion in the correct order.

The second new ingredient is the use of offset frames.  Informally, an offset frame is a shifted
unit-cell decomposition of the line.  By combining the ordinary frame with a suitable admissible
offset frame, one can express a general CPwL forcing term as a finite sum of elementary hats whose
supports avoid the relevant cell seams in at least one frame.  This supplies seam-safe forcing
readouts for the backward affine recursion.  Such a nontrivial admissible offset exists for every
\(M\ge3\); the binary case therefore leads to a more restricted forcing class.

\subsection{Main results}

The homogeneous finite-iterate theorem from \cite{loop} is recalled in
\S\ref{ss:imported-homogeneous-theorem}.  Our main result is its affine linear-depth counterpart,
which may be summarized as follows.

\begin{theorem}[Informal statement of the main result]
\label{thm:intro-affine}
Let \(M\ge3\), and let \(W\) be the refinement operator in \eqref{eq:intro-W}.  Assume that the
associated homogeneous operator preserves a compact support window containing the supports of the
CPwL input \(\gamma\) and forcing term \(B\).  Then, for every \(n\ge1\), the affine iterate
\(W^n\gamma\) has an exact ReLU realization whose width is bounded independently of \(n\) and whose
depth is \(O(n)\).
\end{theorem}

The proof realizes the affine forcing sum in \eqref{eq:intro-affine-decomposition} by a backward
Horner recursion.  The memory controller supplies the residual states in reverse chronological
order, while the offset-frame decomposition supplies seam-safe evaluations of the forcing term.
The remaining selector ambiguity occurs only where the accumulated affine state has already
vanished.  For \(M=2\), the same conclusion holds for the restricted class of ordinary-frame
seam-separated forcing terms defined later.

We also prove a fixed-span stage-dependent extension.  Suppose that
\(W_r\gamma=V\gamma+B_r\), where the homogeneous refinement rule is fixed and
\(B_r=\sum_{\alpha=0}^{N}\lambda_{r,\alpha}B^{(\alpha)}\) belongs to a fixed finite-dimensional
CPwL span.  The same memory-Horner construction then gives exact fixed-width, depth-\(O(n)\)
realizations of \(W_{n-1}\cdots W_1W_0\gamma\).  The weight bounds depend on the fixed refinement
data and on \(\max_{r<n,\alpha}|\lambda_{r,\alpha}|\).

Finally, the anchored-profile, finite-state, and copy-and-connector reductions developed in
\cite{loop} inherit the improved linear-depth bound whenever their forcing sequences lie in a
fixed finite-dimensional CPwL span.  This includes the Hilbert- and Morton-type recursive
constructions discussed there.

\subsection{Structure of the paper}

Section~\ref{sec:preliminaries} fixes the notation, recalls the vectorized refinement formalism and
the homogeneous theorem from \cite{loop}, and introduces offset frames and special basis curves.  
Section~\ref{sec:affine-admissible-frame} proves the core affine forcing result in
one admissible frame.  It constructs the residual memory controller, exact backward replay,
seam-safe forcing readouts, and the memory-Horner recursion.
Section~\ref{sec:general-forcing-affine} combines the ordinary and offset frames to treat arbitrary
compactly supported CPwL forcing for \(M\ge3\), proves the main affine theorem, and records the
restricted binary result.  Section~\ref{sec:stage-dependent-forcing} gives the fixed-span
stage-dependent extension.  Section~\ref{sec:reductions-geometric} records the resulting
linear-depth upgrade for the geometric reductions of \cite{loop}.  The final section summarizes
the scope of the construction and the remaining binary limitation.

\section{Preliminaries and notation}
\label{sec:preliminaries}

This section fixes the notation used in the affine construction.  
We recall the \(M\)-ary digit maps, offset vectorizations, block transition matrices, special basis
curves, and the affine iterate formula.
We also record the homogeneous linear-depth theorem of \cite{loop}, which will be
used as a black box for the term \(V^n\gamma\).

\subsection{Network classes and CPwL functions}

For integers \(W,L,d,N\ge1\), we write
\(\Ups_{W,L}(\ReLU;d,N)\) for the class of outputs of fully connected ReLU networks with width
\(W\), depth \(L\), input dimension \(d\), and output dimension \(N\).  The final realization
theorems have input dimension \(d=1\), although several intermediate controller maps act on fixed
higher-dimensional state spaces.

We use CPwL as shorthand for continuous piecewise linear.  A map \(f:\R^d\to\R^N\) is CPwL if
every compact subset of \(\R^d\) admits a finite polyhedral subdivision on each cell of which \(f\)
is affine.

\begin{remark}
\label{rem:cpwl-basic}
We repeatedly use the following standard facts.
\begin{enumerate}[label=(\roman*),leftmargin=2em]
    \item Every scalar CPwL function on \(\R\) with finitely many breakpoints is realized by a
    one-hidden-layer ReLU network whose width is proportional to its number of affine pieces.
    \item Continuous piecewise affine finite-element functions on finite triangulations in fixed
    dimension admit exact ReLU realizations \cite{relu-fem}.  In particular, uniformly bounded
    triangulation complexity gives uniformly bounded width and depth, and vector-valued maps are
    realized componentwise.
    \item A CPwL map prescribed on a compact finite polyhedral subcomplex of \(\R^d\) admits, after
    a finite subdivision, a compactly supported global CPwL extension.
    \item Compositions of CPwL maps are CPwL.
    \item Precomposition by an affine map can be absorbed into the first affine layer and therefore
    does not change the depth or asymptotic width.
    \item A fixed finite sum of fixed-width, depth-\(O(n)\) networks can be realized with width
    enlarged by only a constant factor and depth still \(O(n)\).
\end{enumerate}
We track network parameters only through coarse exponential bounds in \(n\).  Our primary
complexity measures are therefore width and depth, with bounds on weights and biases recorded
separately.
\end{remark}

\subsection{\texorpdfstring{\(M\)-ary}{M-ary} digit maps, residuals, and offset frames}
\label{ss:m-ary-digits}

Fix an integer \(M\ge2\) and an offset \(s\in[0,1)\).  The corresponding
\emph{\(s\)-offset frame} uses the local coordinate \(x\in[0,1)\) on the physical interval
\([s,s+1)\), and more generally on its translates \([s+k-1,s+k)\).  The case \(s=0\) is called
the \emph{ordinary \(M\)-ary frame}.  We use half-open intervals for the digit dynamics; equalities
between continuous functions are extended to \(x=1\) by continuity.

For \(x\in[0,1)\), define the digit and residual maps by
\[
Q(x)
:=
\lfloor Mx+(M-1)s\rfloor,
\qquad
R(x)
:=
Mx+(M-1)s-Q(x).
\]
Then \(R(x)\in[0,1)\).  The same formula extends \(R\) to a \(1\)-periodic map on \(\R\), and
hence to a map on the circle \(\R/\Z\).  This circle dynamics will provide the geometric model for
the loop and memory controllers.

We call \(s\) \emph{admissible} if \((M-1)s\in\Z\), and then also call the corresponding frame
admissible.  If \(c:=(M-1)s\in\Z\), then
\[
Q(x)
=
\lfloor Mx\rfloor+c,
\qquad
R(x)
=
Mx-\lfloor Mx\rfloor,
\]
and the digit set is \(D:=Q([0,1))=\{c,c+1,\ldots,c+M-1\}\).  Thus admissibility leaves the
ordinary \(M\)-ary residual dynamics unchanged and merely shifts the digit labels by \(c\).  If
\((M-1)s\notin\Z\), the digit set has \(M+1\) elements.

A nontrivial admissible offset exists for every \(M\ge3\), for example \(s=1/(M-1)\).  For
\(M=2\), no such offset exists; accordingly, the binary result below applies only to a restricted
class of forcing terms compatible with the ordinary frame.

Set \(R^0(x):=x\).  For \(j\ge1\), define the successive residuals and digits by
\[
R^j(x)
:=
R\bigl(R^{j-1}(x)\bigr),
\qquad
q_j(x)
:=
Q\bigl(R^{j-1}(x)\bigr).
\]
They satisfy
\[
M R^{j-1}(x)+(M-1)s
=
q_j(x)+R^j(x),
\qquad
j\ge1.
\]

\subsection{The vectorized cascade formalism}

Fix matrices \(A_j\in\R^{p\times p}\), \(j\in\mathbb Z\), with \(A_j=0\) for all but finitely many
\(j\), and define the homogeneous refinement operator by
\begin{equation}
\label{eq:homogeneous-operator}
(V\gamma)(t)
=
\sum_{j\in\mathbb Z}A_j\gamma(Mt-j).
\end{equation}
We work with a fixed support window \([0,L]\), where \(L\ge1\) is an integer, and assume that
\[
\supp f\subset[0,L]
\quad\Longrightarrow\quad
\supp(Vf)\subset[0,L].
\]
The input \(\gamma\) and all forcing terms considered below are supported in this window.

\paragraph{Offset vectorization.}
Fix an offset \(s\in[0,1)\).  The translated unit intervals
\([s+k-1,s+k]\) that meet the support window \([0,L]\) in a set of positive length are indexed by
\[
J_s=
\begin{cases}
\{1,\ldots,L\}, & s=0,\\[1mm]
\{0,1,\ldots,L\}, & 0<s<1.
\end{cases}
\]
We write \(L_s:=|J_s|\), so that \(L_s=L\) in the ordinary frame and \(L_s=L+1\) in a nontrivial
offset frame.

For any \(f:\R\to\R^p\) supported in \([0,L]\), define its \(k\)-th block in the \(s\)-offset frame
by \(f_k^{(s)}(x):=f(x+s+k-1)\), for \(k\in J_s\) and \(x\in[0,1]\).  Its
\emph{\(s\)-offset vectorization} is
\[
\vect_s(f)(x)
:=
\bigl(f_k^{(s)}(x)\bigr)_{k\in J_s}^{\transpose}
\in\R^{pL_s},
\qquad x\in[0,1],
\]
where the blocks are ordered by increasing \(k\).  When the offset frame is fixed, we abbreviate
\(\vect_s(f)\) to \(\vect(f)\).

In particular, set \(G:=\vect_s(\gamma)\) and, for \(n\ge0\),
\[
G^n:=\vect_s(V^n\gamma).
\]
Thus \(G^0=G\), and \(G^n\) records all translated unit-interval pieces of \(V^n\gamma\) in a common
local coordinate.

\paragraph{Block transition matrices.}
For each digit \(q\in D:=Q([0,1))\), define
\(T_q\in\R^{pL_s\times pL_s}\) by prescribing its \(p\times p\) blocks as
\[
(T_q)_{k\ell}
:=
A_{\,q+M(k-1)-(\ell-1)},
\qquad
k,\ell\in J_s.
\]

To explain this definition, fix \(x\in[0,1)\), \(k\in J_s\), and set
\(t:=x+s+k-1\).  If \(q=Q(x)\), then the definition of the residual gives
\[
Mt-j
=
R(x)+s+\bigl(q+M(k-1)-j\bigr).
\]
This argument belongs to the \(\ell\)-th translated block precisely when
\(\ell-1=q+M(k-1)-j\), or equivalently when
\(j=q+M(k-1)-(\ell-1)\).  The corresponding coefficient is therefore
\((T_q)_{k\ell}\).  Integer shifts not represented by \(J_s\) contribute zero by the support
convention.  
Thus the matrix \(T_q\) records the passage from the translated blocks of a function to those of its refinement
on the digit branch \(q\).

\paragraph{Cascade identities.}
The block matrices encode one refinement step as follows.

\begin{proposition}[One-step cascade identity]
For every \(f:\R\to\R^p\) supported in \([0,L]\) and every \(x\in[0,1)\), one has
\[
\vect_s(Vf)(x)
=
T_{q_1(x)}\vect_s(f)(R(x)).
\]
\end{proposition}

\begin{proof}
Fix \(x\in[0,1)\), set \(q:=q_1(x)=Q(x)\), and let \(k\in J_s\).  The \(k\)-th block of the
left-hand side is
\[
\bigl[\vect_s(Vf)(x)\bigr]_k
=
\sum_{j\in\Z}A_j
f\bigl(Mx+Ms+M(k-1)-j\bigr).
\]
Since \(Mx+(M-1)s=R(x)+q\), the argument of \(f\) equals
\(R(x)+s+q+M(k-1)-j\).  Setting
\(\ell:=q+M(k-1)-j+1\), we obtain
\[
\bigl[\vect_s(Vf)(x)\bigr]_k
=
\sum_{\ell\in J_s}
A_{\,q+M(k-1)-(\ell-1)}
\bigl[\vect_s(f)(R(x))\bigr]_\ell.
\]
Terms with \(\ell\notin J_s\) vanish by the support convention.  By the definition of \(T_q\), the
last expression is the \(k\)-th block of \(T_q\vect_s(f)(R(x))\).
\end{proof}

Taking \(f=\gamma\) gives \(G^1(x)=T_{q_1(x)}G(R(x))\).  
Repeated application yields the following full cascade identity.

\begin{corollary}[Iterated cascade identity]
For every \(n\ge1\) and \(x\in[0,1)\), one has
\[
G^n(x)
=
T_{q_1(x)}T_{q_2(x)}\cdots T_{q_n(x)}G(R^n(x)).
\]
\end{corollary}

\begin{proof}
Applying the one-step identity to \(V^{n-1}\gamma\), we have
\(G^n(x)=T_{q_1(x)}G^{n-1}(R(x))\).  The result follows by induction from
\(q_j(R(x))=q_{j+1}(x)\) and \(R^j(R(x))=R^{j+1}(x)\).
\end{proof}

\subsection{Special hats and basis curves}

Let \(0<\mar<\frac12\) be a support margin.  In a one-frame construction it may be chosen
arbitrarily, while in the two-frame construction it will be chosen after the nontrivial offset \(s\)
so that \(2\mar<\min\{s,1-s\}\).

\begin{definition}[Special hat]
A scalar function \(h:\R\to\R\) is called a \emph{special hat} if it is nonnegative and CPwL and
satisfies
\[
\supp h\subset[\mar,1-\mar].
\]
\end{definition}

Let \(e_\mu\in\R^p\), \(\mu=1,\ldots,p\), denote the standard coordinate vectors.  In the fixed
\(s\)-offset frame, a \emph{special basis curve} is a function of the form
\[
\gamma(t)=h(t-\delta)e_\mu,
\qquad
\delta\in\Z+s,
\]
where \(h\) is a special hat.  Write \(\delta=\mm+s\) with \(\mm\in\Z\).  For \(k\in J_s\), its
\(k\)-th block is
\[
\gamma_k^{(s)}(x)
=
h(x+k-1-\mm)e_\mu,
\qquad x\in[0,1].
\]
The support condition on \(h\) implies that all blocks vanish except the one indexed by
\(k=\mm+1\).  Hence, whenever \(\mm+1\in J_s\), the offset vectorization is
\[
\vect_s(\gamma)(x)
=
h(x)u_{\mm+1,\mu},
\]
where \(u_{\mm+1,\mu}\in\R^{pL_s}\) is the coordinate vector corresponding to component \(\mu\) of
block \(\mm+1\).

Thus the offset \(s\) determines the frame in which the scalar hat is read, while the integer
\(\mm=\delta-s\) determines the active vectorization block.  This separates the scalar residual
readout from the matrix-valued cascade update.

\subsection{Affine decomposition and the imported homogeneous theorem}
\label{ss:imported-homogeneous-theorem}

Let \(B:\R\to\R^p\) be compactly supported and CPwL, with
\(\supp B\subset[0,L]\), and define the affine refinement operator by
\begin{equation}
\label{eq:affine-operator}
W\gamma:=V\gamma+B.
\end{equation}
A direct induction gives, for every \(n\ge1\),
\begin{equation}
\label{eq:affine-iterate}
W^n\gamma
=
V^n\gamma+S_n,
\qquad
S_n:=\sum_{r=0}^{n-1}V^rB.
\end{equation}
Thus the problem separates into the homogeneous term \(V^n\gamma\) and the
forcing sum \(S_n\).
The homogeneous contribution is supplied by the following theorem from \cite{loop}.

\begin{theorem}[Homogeneous \(M\)-ary vector-valued realization, \cite{loop}]
\label{thm:imported-homogeneous}
Let \(M\ge2\), and let \(V\) be the homogeneous refinement operator from
\eqref{eq:homogeneous-operator}, with finitely supported matrix mask and invariant support window
\([0,L]\).  If \(\gamma:\R\to\R^p\) is CPwL and supported in \([0,L]\), then there exist constants
\(C_0,C_1>0\), independent of \(n\), such that
\[
V^n\gamma\in\Ups_{C_0,C_1n}(\ReLU;1,p),
\qquad n\ge1.
\]
The realizing networks may moreover be chosen with weights and biases growing at most
exponentially in \(n\).
\end{theorem}

In view of \eqref{eq:affine-iterate}, it therefore remains to construct an exact fixed-width,
depth-\(O(n)\) realization of the forcing sum \(S_n\).

\section{Affine forcing in one admissible frame}
\label{sec:affine-admissible-frame}

In this section we construct a fixed-width, linear-depth realization of the affine forcing sum for
forcing terms decomposed into special basis curves aligned with one admissible frame.  The next
section combines two such framewise constructions
to treat arbitrary compactly supported CPwL forcing when
\(M\ge3\).

Throughout the section, fix an admissible offset \(s\in[0,1)\) and set
\(c:=(M-1)s\in\Z\).  By the discussion in \S\ref{ss:m-ary-digits}, admissibility gives
\(Q(x)=\lfloor Mx\rfloor+c\) and
\(R(x)=Mx-\lfloor Mx\rfloor\) for \(x\in[0,1)\).  We therefore define
\[
d_j(x):=q_j(x)-c\in\{0,\ldots,M-1\},
\qquad
\mathsf T_r:=T_{c+r},
\quad r=0,\ldots,M-1.
\]
Then \(T_{q_j(x)}=\mathsf T_{d_j(x)}\) for every \(j\ge1\).

\subsection{The vectorized forcing sum and Horner recursion}

Let \(B:\R\to\R^p\) be CPwL and supported in \([0,L]\), and write
\(b:=\vect_s(B)\).  We assume that \(b\) admits a finite frame-aligned special-hat expansion
\begin{equation}
\label{eq:frame-aligned-vectorized-forcing}
b(x)
=
\sum_{\nu=1}^{N_B}a_\nu h_\nu(x)v_\nu,
\qquad
\supp h_\nu\subset[\mar,1-\mar],
\end{equation}
where \(a_\nu\in\R\), \(v_\nu\in\R^{pL_s}\), and each \(h_\nu\) is a special hat.  In
particular, if \(B(t)=h(t-\delta)e_\mu\) is a special basis curve with
\(\delta\in\Z+s\), then the expansion has one term and \(v_\nu\) is the corresponding coordinate
vector of the active block.

For \(n\ge1\), set \(S_n:=\sum_{r=0}^{n-1}V^rB\) and
\(G_n:=\vect_s(S_n)\).  
For \(k\ge0\), the cascade identity gives
\[
\vect_s(V^kB)(x)
=
\mathsf T_{d_1(x)}\cdots\mathsf T_{d_k(x)}b(R^kx),
\qquad x\in[0,1),
\]
where the empty product is the identity.  Summing over \(k\), we obtain
\[
G_n(x)
=
\sum_{k=0}^{n-1}
\mathsf T_{d_1(x)}\cdots\mathsf T_{d_k(x)}b(R^kx).
\]

This sum is evaluated efficiently by a backward Horner recursion.  For fixed \(n\), define
\(U_{n-1}(x):=b(R^{n-1}x)\), and then set
\begin{equation}
\label{eq:horner-recursion}
U_j(x)
:=
b(R^jx)+\mathsf T_{d_{j+1}(x)}U_{j+1}(x),
\qquad
j=n-2,\ldots,0.
\end{equation}

\begin{lemma}[Horner form of the forcing sum]
\label{lem:horner-form}
For every \(n\ge1\) and \(x\in[0,1)\), one has
\[
U_0(x)=G_n(x)=\vect_s(S_n)(x).
\]
\end{lemma}

\begin{proof}
Expanding the recursion gives
\[
U_0(x)
=
b(x)
+
\sum_{k=1}^{n-1}
\mathsf T_{d_1(x)}\cdots\mathsf T_{d_k(x)}b(R^kx),
\]
which is the preceding expression for \(G_n(x)\).
\end{proof}

The following elementary consequence will make the later selector ambiguity harmless.

\begin{lemma}[Vanishing accumulated state]
\label{lem:vanishing-accumulated-state}
Fix \(j\in\{0,\ldots,n-2\}\).  If \(b(R^ix)=0\) for every
\(i=j+1,\ldots,n-1\), then \(U_{j+1}(x)=0\).
\end{lemma}

\begin{proof}
Starting from \(U_{n-1}(x)=b(R^{n-1}x)=0\), the conclusion follows by backward induction from
\eqref{eq:horner-recursion}.
\end{proof}

\subsection{The residual memory controller}
\label{subsec:residual-memory-controller}

The affine Horner recursion requires the residual states in reverse chronological order.  Since the
residual map is \(M\)-to-one on the circle, \(R^j x\) cannot be recovered continuously from
\(R^{j+1}x\) alone.  We therefore augment the residual loop by a memory coordinate and replace its
noninjective dynamics by an injective skew-product.

Let \(\mathbb T:=\R/\Z\), and choose a simple polygonal embedding
\(E:\mathbb T\to\Gamma\subset[-1,1]^2\).  We take \(E\) to be CPwL with respect to a finite
subdivision of the circle.  Define the inverse-branch separation constant by
\begin{equation}
\label{eq:branch-separation}
\Delta_E
:=
\min_{\substack{1\le a\le M-1\\ t\in\mathbb T}}
\|E(t+a/M)-E(t)\|_\infty .
\end{equation}
For each \(a=1,\ldots,M-1\), the points \(t\) and \(t+a/M\) are distinct on \(\mathbb T\).
Since \(E\) is an embedding and the minimum is taken over a compact set, one has
\(\Delta_E>0\).

Following \cite{loop}, let the degree-\(M\) circle map \(t\mapsto Mt \pmod 1\) induce
\(F_\Gamma:\Gamma\to\Gamma\) through
\[
F_\Gamma(E(t)):=E(Mt).
\]
After subdividing \(\Gamma\) at the images under \(E\) of the breakpoints of \(E\) and their
preimages under the degree-\(M\) circle map, \(F_\Gamma\) is affine on each resulting edge.
Extending this finite polyhedral complex to a triangulation of a sufficiently large polygon gives a
global CPwL extension \(F:\R^2\to\R^2\).

Set \(C:=[-1,1]^2\) and \(X:=\Gamma\times C\subset\R^4\).  Choose
\(\alpha,\beta>0\) so that
\begin{equation}
\label{eq:memory-parameters}
\alpha+\beta\le1,
\qquad
2\alpha<\beta\Delta_E.
\end{equation}
Such a choice is possible by fixing any \(\beta\in(0,1)\) and then choosing
\[
0<\alpha<\min\{1-\beta,\beta\Delta_E/2\}.
\]
Define the memory update by
\begin{equation}
\label{eq:memory-map}
\mathcal Q(r,y)
:=
\bigl(F(r),\,\beta r+\alpha y\bigr),
\qquad
(r,y)\in\R^2\times\R^2.
\end{equation}
For \(x\in[0,1]\), define the valid memory states by
\[
Z_0(x):=(E(x),0),
\qquad
Z_j(x):=\mathcal Q^j(Z_0(x)).
\]
In particular, the first coordinate of \(Z_j(x)\) is \(E(M^jx)\), with the argument understood
modulo \(1\).

\begin{lemma}[Injective residual memory map]
\label{lem:memory-injective}
The memory update has the following properties.
\begin{enumerate}[label=(\roman*),leftmargin=2em]
    \item \(\mathcal Q(X)\subset X\).
    \item The restriction \(\mathcal Q|_X\) is injective.
    \item The restriction \(\mathcal Q|_X\) is piecewise affine on a finite polyhedral complex.
    \item Its inverse on \(\mathcal Q(X)\) is piecewise affine and admits a global CPwL extension
    \(P:\R^4\to\R^4\).  In particular,
    \[
    P(\mathcal Q(Z))=Z,
    \qquad Z\in X.
    \]
\end{enumerate}
\end{lemma}

\begin{proof}
Let \((r,y)\in X\).  Since \(r\in\Gamma\), one has \(F(r)=F_\Gamma(r)\in\Gamma\).  Moreover,
\[
\|\beta r+\alpha y\|_\infty
\le
\beta\|r\|_\infty+\alpha\|y\|_\infty
\le
\alpha+\beta
\le1,
\]
and hence \(\mathcal Q(r,y)\in X\).  This proves (i).

To prove injectivity, suppose that
\(\mathcal Q(E(t),y)=\mathcal Q(E(u),\widetilde y)\).  Equality of the first coordinates gives
\(E(Mt)=E(Mu)\).  Since \(E\) is injective on \(\mathbb T\), it follows that
\(u\equiv t+a/M\pmod1\) for some \(a\in\{0,\ldots,M-1\}\).

If \(a=0\), then \(E(u)=E(t)\), and equality of the second coordinates gives
\(y=\widetilde y\).  If \(a\ne0\), the definition of \(\Delta_E\) and equality of the memory
coordinates imply
\[
\beta\Delta_E
\le
\beta\|E(t)-E(u)\|_\infty
=
\alpha\|\widetilde y-y\|_\infty
\le
2\alpha,
\]
contradicting \eqref{eq:memory-parameters}.  Thus \(a=0\), proving (ii).

For (iii), subdivide \(\Gamma\) so that \(F_\Gamma\) is affine on every edge, triangulate \(C\),
and triangulate the resulting product cells in \(X=\Gamma\times C\).  On every simplex of this
finite complex, both \(r\mapsto F(r)\) and \((r,y)\mapsto\beta r+\alpha y\) are affine.

Finally, refine this complex if necessary so that \(\mathcal Q\) is affine on every simplex.
Because \(\mathcal Q|_X\) is globally injective, the images of two simplices intersect only in the
image of their intersection.  The image simplices therefore form a finite polyhedral complex on
\(\mathcal Q(X)\), and the inverse is affine on each image simplex.  These affine pieces agree on
their common faces, so the inverse \(P:\mathcal Q(X)\to X\) is piecewise affine.

Choose a sufficiently large box \(K\subset\R^4\) containing \(\mathcal Q(X)\) in its interior, and
extend the image complex to a finite triangulation of \(K\).  Assign the prescribed values on the
vertices of \(\mathcal Q(X)\), assign zero values on the boundary vertices of \(K\), and choose
arbitrary values at the remaining interior vertices.  Affine extension over the simplices, followed
by the zero extension outside \(K\), gives a global CPwL map \(P:\R^4\to\R^4\).
\end{proof}

The following is immediate.

\begin{corollary}[Exact backward replay]
\label{cor:backward-playback}
For every \(0\le j\le n\), one has
\[
P^j(Z_n(x))=Z_{n-j}(x).
\]
\end{corollary}

\subsection{Lifted forcing readouts}

The loop state records the residual point on \(\mathbb T\), but no single continuous scalar readout
on \(\Gamma\) can recover its representative in \([0,1]\) across the seam.  Following the
loop-coordinate construction of \cite{loop}, we therefore use two complementary seam-compatible
readouts.

Fix \(0<\varepsilon<\mar\), and define \(r^-,r^+:[0,1]\to[0,1]\) by
\[
r^-(t)
=
\begin{cases}
t, & 0\le t\le1-\varepsilon,\\[1mm]
\dfrac{1-\varepsilon}{\varepsilon}(1-t),
   & 1-\varepsilon\le t\le1,
\end{cases}
\]
and
\[
r^+(t)
=
\begin{cases}
1-\dfrac{1-\varepsilon}{\varepsilon}t,
   & 0\le t\le\varepsilon,\\[1mm]
t, & \varepsilon\le t\le1.
\end{cases}
\]
Since \(r^-(0)=r^-(1)=0\) and \(r^+(0)=r^+(1)=1\), these functions define CPwL readouts
\(\rho^\pm:\Gamma\to[0,1]\) through
\(\rho^\pm(E(t)):=r^\pm(t)\).  We fix global CPwL extensions
\(\rho^\pm:\R^2\to\R\).

For a special hat \(h\), define the lifted scalar readout
\[
H_h(r,y)
:=
\min\bigl\{h(\rho^-(r)),h(\rho^+(r))\bigr\},
\qquad
(r,y)\in\R^2\times\R^2.
\]

\begin{lemma}[Lifted special-hat readout]\label{lem:lifted-special-hat-readout}
The map \(H_h:\R^4\to\R\) is CPwL and satisfies
\[
H_h(E(t),y)=h(t),
\qquad
t\in[0,1],\quad y\in\R^2.
\]
Consequently, for every \(x\in[0,1)\) and \(j\ge0\), one has
\[
H_h(Z_j(x))=h(R^jx).
\]
\end{lemma}

\begin{proof}
The map \(H_h\) is CPwL because compositions and pointwise minima of scalar CPwL functions are
CPwL.  If \(t\in[\varepsilon,1-\varepsilon]\), then
\(r^-(t)=r^+(t)=t\), and the identity is immediate.

If \(t\in[0,\varepsilon]\), then \(r^-(t)=t\) and \(h(t)=0\), since
\(\varepsilon<\mar\) and \(\supp h\subset[\mar,1-\mar]\).  The nonnegativity of \(h\) therefore
gives
\[
\min\bigl\{h(r^-(t)),h(r^+(t))\bigr\}=0=h(t).
\]
The same argument applies on \([1-\varepsilon,1]\), using \(r^+(t)=t\).
This proves the first identity.

The first coordinate of \(Z_j(x)\) is \(E(M^jx)=E(R^jx)\).  Substitution into the first identity
therefore yields the valid-state formula.
\end{proof}

Using the frame-aligned expansion
\eqref{eq:frame-aligned-vectorized-forcing}, define
\[
\mathcal B(Z)
:=
\sum_{\nu=1}^{N_B}a_\nu H_{h_\nu}(Z)v_\nu,
\qquad
Z\in\R^4.
\]

\begin{corollary}[Exact lifted forcing readout]
\label{cor:exact-lifted-forcing-readout}
The map \(\mathcal B:\R^4\to\R^{pL_s}\) is CPwL and, for every
\(x\in[0,1)\) and \(j\ge0\), satisfies
\[
\mathcal B(Z_j(x))=b(R^jx).
\]
\end{corollary}

\begin{proof}
The claim follows by applying
\Cref{lem:lifted-special-hat-readout} termwise in
\eqref{eq:frame-aligned-vectorized-forcing}.
\end{proof}

\subsection{Selectors and the selected matrix action}
\label{subsec:memory-selectors}

The digit indicators are discontinuous at the \(M\)-ary breakpoints and therefore cannot be
represented exactly by continuous loop readouts.  We replace them by CPwL selectors that are exact
outside small transition intervals.  On those intervals, exactness of the matrix action will follow
from the vanishing of the accumulated Horner state.

Fix \(0<\bar\delta<1\).  For a prescribed final depth \(n\ge1\), set
\(\delta_n:=\bar\delta\,\mar M^{-(n+1)}\) and define the transition set
\[
J_n
:=
\bigcup_{k=0}^{M-1}
\left[\frac{k}{M},\frac{k}{M}+\delta_n\right].
\]
Since \(\delta_n<1/M\), these intervals are pairwise disjoint and contained in \([0,1)\).

Let
\(\vartheta^{(n)}=(\vartheta^{(n)}_0,\ldots,\vartheta^{(n)}_{M-1})\) be the continuous
piecewise affine map from \([0,1]\) into the standard simplex defined as follows.  On each
transition interval \([k/M,k/M+\delta_n]\), it interpolates linearly from the coordinate vector
corresponding to \(k-1\pmod M\) to the coordinate vector corresponding to \(k\).  Between
successive transition intervals, it is equal to the coordinate vector corresponding to the current
digit.  Thus we have
\[
\sum_{r=0}^{M-1}\vartheta^{(n)}_r(t)=1,
\qquad
0\le \vartheta^{(n)}_r(t)\le1,
\]
and, whenever \(t\in[0,1)\setminus J_n\), one has
\[
\vartheta^{(n)}_r(t)
=
\begin{cases}
1, & r=\lfloor Mt\rfloor,\\
0, & r\ne\lfloor Mt\rfloor.
\end{cases}
\]
At the endpoints, this construction gives
\(\vartheta^{(n)}(0)=\vartheta^{(n)}(1)=e_{M-1}\).  Hence the selectors descend continuously to
the loop through
\[
\chi^{(n)}_r(E(t))
:=
\vartheta^{(n)}_r(t).
\]
We fix global CPwL extensions to \(\R^2\), and then lift them to the memory space by setting
\(\chi^{(n)}_r(r,y):=\chi^{(n)}_r(r)\).

We use the following elementary exact gate, which is a simpler variant of the product gadget
introduced in \cite[Lemma~9]{source} and used in \cite{loop}.  For \(a>0\),
\(\lambda\in\R\), and \(Y\in\R^{pL_s}\), define
\[
\Pi_a(\lambda,Y)
:=
\ReLU\bigl(Y-a(1-\lambda)\mathbf1\bigr)
-
\ReLU\bigl(-Y-a(1-\lambda)\mathbf1\bigr),
\]
where \(\mathbf1\in\R^{pL_s}\) denotes the all-ones vector, and all ReLU operations are applied componentwise.
For
\(\lambda\in[0,1]\) and \(Y\in[-a,a]^{pL_s}\), one has
\[
\Pi_a(1,Y)=Y,
\qquad
\Pi_a(0,Y)=0,
\qquad
\Pi_a(\lambda,0)=0.
\]

To choose the scale, set
\[
B_*:=\sup_{t\in[0,1]}\|b(t)\|_\infty,
\qquad
\tau:=\max_{0\le r\le M-1}
\|\mathsf T_r\|_{\infty\to\infty}.
\]
The Horner recursion gives
\[
\|U_j(x)\|_\infty
\le
B_*\sum_{\ell=0}^{n-1-j}\tau^\ell.
\]
We may therefore take
\[
a_n
:=
1+\tau B_*\sum_{\ell=0}^{n-1}\tau^\ell,
\]
so that \(\|\mathsf T_rU_j(x)\|_\infty\le a_n\) for every \(r,j\), and \(x\).  In particular,
there are constants \(C,\Lambda>0\), depending only on the fixed data, such that
\(a_n\le C\Lambda^n\).

Define the selected matrix action by
\[
\mathcal T_n(Z,U)
:=
\sum_{r=0}^{M-1}
\Pi_{a_n}\bigl(\chi^{(n)}_r(Z),\mathsf T_rU\bigr).
\]

\begin{lemma}[Exact selected matrix action on Horner states]
\label{lem:selected-matrix-action}
For every \(j=0,\ldots,n-2\) and \(x\in[0,1)\), one has
\[
\mathcal T_n\bigl(Z_j(x),U_{j+1}(x)\bigr)
=
\mathsf T_{d_{j+1}(x)}U_{j+1}(x).
\]
\end{lemma}

\begin{proof}
Set \(t:=R^jx\).  If \(t\notin J_n\), then the selectors are exact and
\(d_{j+1}(x)=\lfloor Mt\rfloor\).  Exactly one selector equals \(1\), while all the others vanish,
so the gating identities give
\[
\mathcal T_n\bigl(Z_j(x),U_{j+1}(x)\bigr)
=
\mathsf T_{d_{j+1}(x)}U_{j+1}(x).
\]

Suppose now that \(t\in J_n\).  Then
\(t=k/M+\delta\) for some \(k\in\{0,\ldots,M-1\}\) and
\(0\le\delta\le\delta_n\).  Since \(\delta_n<1/M\), the first residual is
\(R(t)=M\delta\).  Moreover, for every \(i=j+1,\ldots,n-1\), no wrap-around occurs and
\[
R^ix
=
M^{\,i-j}\delta
\le
M^{n-1}\delta_n
=
\bar\delta\,\mar M^{-2}
<
\mar.
\]
Every special hat in
\eqref{eq:frame-aligned-vectorized-forcing} vanishes on \([0,\mar)\), and hence
\(b(R^ix)=0\) for \(i=j+1,\ldots,n-1\).  By
\Cref{lem:vanishing-accumulated-state}, it follows that \(U_{j+1}(x)=0\).

Consequently, \(\mathsf T_rU_{j+1}(x)=0\) for every branch \(r\).  The identity
\(\Pi_a(\lambda,0)=0\) then gives
\[
\mathcal T_n\bigl(Z_j(x),U_{j+1}(x)\bigr)
=
0
=
\mathsf T_{d_{j+1}(x)}U_{j+1}(x),
\]
which completes the proof.
\end{proof}

\subsection{The backward memory-Horner network}

We now combine exact backward replay, the lifted forcing readout, and the selected matrix action.
First compute the terminal memory state \(Z_n(x)=\mathcal Q^n(Z_0(x))\), and initialize
\[
\widehat Z_{n-1}(x):=P(Z_n(x)),
\qquad
\widehat U_{n-1}(x):=\mathcal B(\widehat Z_{n-1}(x)).
\]
For \(j=n-2,\ldots,0\), define the backward recursion by
\[
\widehat Z_j(x):=P(\widehat Z_{j+1}(x)),
\qquad
\widehat U_j(x)
:=
\mathcal B(\widehat Z_j(x))
+
\mathcal T_n(\widehat Z_j(x),\widehat U_{j+1}(x)).
\]

\begin{lemma}[Exactness of the backward memory-Horner recursion]
\label{lem:memory-horner-exactness}
For every \(j=0,\ldots,n-1\) and \(x\in[0,1)\), one has
\[
\widehat Z_j(x)=Z_j(x),
\qquad
\widehat U_j(x)=U_j(x).
\]
In particular, \(\widehat U_0(x)=\vect_s(S_n)(x)\).
\end{lemma}

\begin{proof}
By exact backward replay, we have
\(\widehat Z_{n-1}(x)=P(Z_n(x))=Z_{n-1}(x)\).  The lifted forcing readout then gives
\[
\widehat U_{n-1}(x)
=
\mathcal B(Z_{n-1}(x))
=
b(R^{n-1}x)
=
U_{n-1}(x).
\]

Suppose that the two identities hold at level \(j+1\).  Exact backward replay gives
\(\widehat Z_j(x)=P(Z_{j+1}(x))=Z_j(x)\).  Using
\Cref{cor:exact-lifted-forcing-readout,lem:selected-matrix-action}, we therefore obtain
\[
\begin{aligned}
\widehat U_j(x)
&=
\mathcal B(Z_j(x))
+
\mathcal T_n(Z_j(x),U_{j+1}(x)) \\
&=
b(R^jx)
+
\mathsf T_{d_{j+1}(x)}U_{j+1}(x)
=
U_j(x).
\end{aligned}
\]
Backward induction proves the claim, and the final identity follows from
\Cref{lem:horner-form}.
\end{proof}

\begin{theorem}[Vectorized forcing realization in one admissible frame]
\label{thm:fixed-frame-vectorized-forcing}
Fix an admissible \(s\)-offset frame, and let \(B:\R\to\R^p\) be CPwL and supported in
\([0,L]\).  Suppose that \(b=\vect_s(B)\) has a frame-aligned expansion of the form
\eqref{eq:frame-aligned-vectorized-forcing}.
Set \(S_n:=\sum_{r=0}^{n-1}V^rB\).  Then there exist constants
\(C_0,C_1>0\), independent of \(n\), and networks
\(\Phi_n\in\Ups_{C_0,C_1n}(\ReLU;1,pL_s)\) such that
\[
\Phi_n(x)=\vect_s(S_n)(x),
\qquad
x\in[0,1].
\]
The realizing networks may moreover be chosen with weights and biases growing at most
exponentially in \(n\).
\end{theorem}

\begin{proof}
Fix a global CPwL extension of the map \(x\mapsto(E(x),0)\) from \([0,1]\) to \(\R\).
The input is first mapped to \(Z_0(x)=(E(x),0)\), after which \(n\) copies of the fixed memory
update \(\mathcal Q\) produce \(Z_n(x)\).  The backward computation uses \(n\) stage blocks formed
from \(P\), \(\mathcal B\), and \(\mathcal T_n\).  All state dimensions are fixed: the memory
state lies in \(\R^4\), while the accumulated affine state lies in \(\R^{pL_s}\).

The maps \(\mathcal Q\), \(P\), and \(\mathcal B\) are fixed CPwL maps on fixed-dimensional
spaces and hence have exact ReLU realizations of fixed width and depth.  The selected-action map
\(\mathcal T_n\) also has an architecture independent of \(n\): the number of selectors and their
affine pieces is fixed, and the product gadget has fixed size.  Its \(n\)-dependence enters only
through the selector slopes, which are proportional to \(\delta_n^{-1}\), and through the scale
\(a_n\).  Since \(\delta_n^{-1}=O(M^n)\) and \(a_n\le C\Lambda^n\), these parameters grow at most
exponentially in \(n\).

By \Cref{lem:memory-horner-exactness}, the output \(\Phi_n\) of the backward recursion satisfies
\(\Phi_n(x)=\vect_s(S_n)(x)\) for \(x\in[0,1)\).
Both the network output and \(\vect_s(S_n)\) are continuous on
\([0,1]\), so the equality also holds at \(x=1\).  The forward and backward phases use \(O(n)\)
fixed-size blocks in total.  Consequently, the width is bounded independently of \(n\), the depth
is \(O(n)\), and the parameter bound is at most exponential.
\end{proof}

\subsection{From vectorized blocks to the global forcing curve}
\label{subsec:fixed-frame-global-forcing}

Let \(S_{n,k}:[0,1]\to\R^p\), \(k\in J_s\), denote the blocks of
\(\vect_s(S_n)\).  Then, on the translated interval associated with \(k\), one has
\[
S_n(t)
=
S_{n,k}(t-s-k+1),
\qquad
t\in[s+k-1,s+k].
\]
Because \(S_n\) is continuous, the neighboring block formulas agree at their common endpoints.
These finitely many block realizations can be assembled into a global realization by the standard
CPwL gluing construction used in \cite{loop}.

\begin{corollary}[Fixed-frame forcing realization on \(\R\)]
\label{cor:fixed-frame-global-forcing}
Under the assumptions of \Cref{thm:fixed-frame-vectorized-forcing}, there exist constants
\(C_0,C_1>0\), independent of \(n\), such that
\[
S_n
=
\sum_{r=0}^{n-1}V^rB
\in
\Ups_{C_0,C_1n}(\ReLU;1,p),
\qquad
n\ge1.
\]
These networks may again be chosen with weights and biases growing at most exponentially
in \(n\).
\end{corollary}

\begin{proof}
By \Cref{thm:fixed-frame-vectorized-forcing}, all blocks \(S_{n,k}\) are realized simultaneously
by a fixed-width, depth-\(O(n)\) network.  For each \(k\in J_s\), precomposition with the affine map
\(t\mapsto t-s-k+1\) produces the corresponding realization on
\([s+k-1,s+k]\).

The block-to-global gluing construction of \cite{loop} combines these finitely many translated
outputs and sets the result to zero outside the support window.  Since the number \(L_s\) of blocks
is fixed, this requires only a fixed enlargement of the width and an \(O(1)\) increase in depth.
The resulting network therefore has fixed width and depth \(O(n)\), while its parameter bound
remains exponential after adjusting the constants.
\end{proof}

\section{General forcing and the affine theorem}
\label{sec:general-forcing-affine}

Section~\ref{sec:affine-admissible-frame} treated forcing terms aligned with a single admissible
frame.  For \(M\ge3\), we now combine the ordinary frame with a nontrivial admissible offset frame
and decompose arbitrary compactly supported CPwL forcing into atoms aligned with one of the two.
The one-frame theorem then applies atom by atom, and combining the resulting forcing realization
with \Cref{thm:imported-homogeneous} yields the main affine theorem.

Throughout this section, fix \(s\in(0,1)\) such that \((M-1)s\in\Z\), and choose
\(\mar>0\) so that \(2\mar<\min\{s,1-s\}\).  Thus both the ordinary frame and the
\(s\)-offset frame are admissible.

\subsection{Two-frame special-hat decomposition}

The following nodal decomposition expresses arbitrary compactly supported CPwL data as a finite sum
of special hats whose translations lie in \(\Lambda_s:=\Z\cup(\Z+s)\), and hence are aligned with one of the two
admissible frames.

\begin{lemma}[Two-frame special-hat decomposition]
\label{lem:finite-hat-decomposition}
Let \(F:\R\to\R^p\) be compactly supported and CPwL.  Then \(F\) admits a representation
\[
F(t)
=
\sum_{\nu=1}^{N_*}
a_\nu h_\nu(t-\delta_\nu)e_{\mu_\nu},
\]
where \(a_\nu\in\R\), \(\delta_\nu\in\Lambda_s\),
\(\mu_\nu\in\{1,\ldots,p\}\), and each \(h_\nu\) is a special hat with at most three
breakpoints.

More precisely, suppose that \(\supp F\subset[a,b]\), with \(a<b\), and that the union of the
breakpoint sets of its scalar components contains at most \(m\) points in \((a,b)\).  Setting
\(\eta:=\min\{s,1-s\}-2\mar>0\), one may choose the decomposition so that
\(N_*\le p\bigl(\lfloor2(b-a)/\eta\rfloor+m+2\bigr)\).
\end{lemma}

\begin{proof}
Choose a common partition of \([a,b]\) containing all breakpoints of the components of \(F\).
If the lengths of its original intervals are \(d_i\), subdivide the \(i\)-th interval into
\(\lfloor2d_i/\eta\rfloor+1\) equal pieces.  Every resulting mesh interval then has length
strictly less than \(\eta/2\), and the total number of mesh nodes is at most
\[
\left\lfloor\frac{2(b-a)}{\eta}\right\rfloor+m+2.
\]
On this refined grid, \(F\) has the vector-valued nodal expansion
\(F(t)=\sum_i\psi_i(t)F(t_i)\), where each \(\psi_i\) is a nonnegative nodal hat.  Splitting the
vectors \(F(t_i)\) into their coordinate components gives at most \(p\) scalar-coordinate terms
per mesh node.

Fix one of the nodal hats \(\psi\), and write
\(\supp\psi=[a_\psi,b_\psi]\).  Since an interior nodal hat spans at most two adjacent mesh
intervals, its support has length \(b_\psi-a_\psi<\eta\).  The translations \(\delta\) for which
\(h(t):=\psi(t+\delta)\) satisfies \(\supp h\subset[\mar,1-\mar]\) form the interval
\[
I_\psi
:=
[\,b_\psi+\mar-1,\ a_\psi-\mar\,].
\]
Its length satisfies
\[
|I_\psi|
=
1-(b_\psi-a_\psi)-2\mar
>
\max\{s,1-s\}.
\]
The successive gaps in \(\Lambda_s=\Z\cup(\Z+s)\) have lengths \(s\) and \(1-s\).  Hence
\(I_\psi\) contains some \(\delta\in\Lambda_s\).  For this choice, \(h\) is a special hat and
\(\psi(t)=h(t-\delta)\).  Since \(\psi\) is a nodal hat, \(h\) has at most three breakpoints.
Applying this construction to every nodal term and coordinate yields the asserted decomposition
and the stated bound on \(N_*\).
\end{proof}

\subsection{General compactly supported forcing}
\label{subsec:general-compact-forcing}

We now treat arbitrary compactly supported CPwL forcing terms.  The two-frame decomposition aligns
each forcing atom with either the ordinary frame or the fixed admissible \(s\)-offset frame.

\begin{theorem}[General compactly supported forcing]
\label{thm:general-forcing}
Let \(B:\R\to\R^p\) be CPwL and supported in \([0,L]\), and suppose that 
the union of the breakpoint sets of its scalar components contains at most \(m\) points in
\((0,L)\).  For the admissible offset
\(s\) fixed above, there exist constants \(C_0,C_1>0\), independent of \(n\), such that
\[
S_n[B]
:=
\sum_{r=0}^{n-1}V^rB
\in
\Ups_{C_0,C_1n}(\ReLU;1,p),
\qquad
n\ge1.
\]
The architectural constants depend only on \(M,\mar,p,L,m,s\) and the fixed mask
\((A_j)_j\).  The realizing networks may moreover be chosen with weights and biases bounded by
\(C_B\Lambda^n\), where \(\Lambda>0\) depends only on the fixed refinement and frame data, while
\(C_B>0\) may also depend on the CPwL coefficients of \(B\).
\end{theorem}

\begin{proof}
By \Cref{lem:finite-hat-decomposition}, the forcing term has a representation
\begin{equation*}
B(t)
=
\sum_{\nu=1}^{N_*}
a_\nu h_\nu(t-\delta_\nu)e_{\mu_\nu},
\end{equation*}
where \(a_\nu\in\R\), \(\delta_\nu\in\Z\cup(\Z+s)\),
\(\mu_\nu\in\{1,\ldots,p\}\), and each \(h_\nu\) is a special hat with at most three
breakpoints.  The number \(N_*\) is bounded in terms of \(p,L,m,s\), and \(\mar\), independently
of \(n\).

The forcing sum is linear in \(B\), and hence we have
\[
S_n[B]
=
\sum_{\nu=1}^{N_*}
a_\nu
S_n\!\left[h_\nu(\,\cdot-\delta_\nu)e_{\mu_\nu}\right].
\]
If \(\delta_\nu\in\Z\), the corresponding atom is aligned with the ordinary frame; if
\(\delta_\nu\in\Z+s\), it is aligned with the \(s\)-offset frame.  Both frames are admissible, so
\Cref{cor:fixed-frame-global-forcing} applies to every summand.

The number of summands is bounded independently of \(n\).  Their realizing networks can therefore
be run in parallel and combined by a final affine layer.  This enlarges the width by only a fixed
factor and preserves depth \(O(n)\).  The coefficients \(a_\nu\) and the CPwL data of the special hats affect
only the network parameters, so the same finite combination gives a bound of the form
\(C_B\Lambda^n\).
\end{proof}

\subsection{The affine main theorem}
\label{subsec:affine-main-theorem}

We now combine the general forcing theorem with the homogeneous realization theorem imported from
\cite{loop}.

\begin{theorem}[Affine realization theorem]
\label{thm:affine-main}
Let \(M\ge3\), and consider the affine refinement operator
\[
(W\gamma)(t)
=
\sum_{j\in\Z}A_j\gamma(Mt-j)+B(t),
\]
where the matrix mask \((A_j)_{j\in\Z}\subset\R^{p\times p}\) is finitely supported and the
associated homogeneous operator \(V\) preserves the support window \([0,L]\).  Suppose that
\(\gamma,B:\R\to\R^p\) are CPwL and supported in \([0,L]\).  Then there exist constants
\(C_0,C_1>0\), independent of \(n\), such that
\[
W^n\gamma
\in
\Ups_{C_0,C_1n}(\ReLU;1,p),
\qquad
n\ge1.
\]
The constants depend only on the fixed refinement data and the CPwL complexity of \(\gamma\) and
\(B\).  The realizing networks may moreover be chosen with weights and biases bounded by
\(C\Lambda^n\) for suitable constants \(C,\Lambda>0\).
\end{theorem}

\begin{proof}
With the admissible offset and support margin fixed above, the affine iterate identity gives
\[
W^n\gamma
=
V^n\gamma+S_n[B].
\]
The homogeneous theorem \Cref{thm:imported-homogeneous} gives a fixed-width, depth-\(O(n)\)
realization of \(V^n\gamma\), while \Cref{thm:general-forcing} gives one for \(S_n[B]\).
Running the two networks in parallel and adding their outputs in a final affine layer preserves
fixed width up to a constant factor and preserves depth \(O(n)\).  The exponential parameter bound
is preserved after enlarging the constants.
\end{proof}

\begin{definition}[Ordinary-frame seam-separated forcing]
\label{def:binary-seam-separated}
A forcing term \(B:\R\to\R^p\) supported in \([0,L]\) is called
\emph{ordinary-frame seam-separated} if there exist
\(\varrho\in(0,\frac12)\), scalars \(a_\nu\), vectors \(v_\nu\in\R^{pL}\), and nonnegative
CPwL functions \(h_\nu\) such that
\[
\vect_0(B)(x)
=
\sum_{\nu=1}^{N_B}a_\nu h_\nu(x)v_\nu,
\qquad
\supp h_\nu\subset[\varrho,1-\varrho].
\]
\end{definition}

\begin{corollary}[Restricted binary affine realization]
\label{cor:binary-affine}
Let \(M=2\), let the matrix mask be finitely supported, and assume that the associated homogeneous
operator \(V\) preserves \([0,L]\).  Let \(\gamma,B:\R\to\R^p\) be CPwL and supported in
\([0,L]\).  If \(B\) is ordinary-frame seam-separated, then there exist constants
\(C_0,C_1>0\), independent of \(n\), such that
\[
W^n\gamma
\in
\Ups_{C_0,C_1n}(\ReLU;1,p),
\qquad
n\ge1.
\]
The realizing networks may be chosen with weights and biases growing at most exponentially
in \(n\).
\end{corollary}

\begin{proof}
Choose a margin \(\varrho\) and a representation as in
\Cref{def:binary-seam-separated}.  Apply
\Cref{cor:fixed-frame-global-forcing} in the ordinary binary frame using this margin, and apply
\Cref{thm:imported-homogeneous} to the homogeneous term.  Combining the two realizations as in the
proof of \Cref{thm:affine-main} gives the result.
\end{proof}

The memory controller supplies the residual states in the reverse order required by the affine
recursion.  The offset-frame decomposition supplies forcing readouts away from the seams, while the
remaining selector ambiguity occurs only where the accumulated Horner state has already vanished.
For \(M\ge3\), two admissible frames suffice for arbitrary compactly supported CPwL forcing.  For
\(M=2\), the present offset-frame argument yields only the restricted class above, although some
important seam-touching systems admit separate linear-depth constructions.

\begin{remark}[The Takagi recursion]
\label{rem:takagi-forward-accumulation}
Consider the binary affine recursion
\[
(\mathcal Wg)(x)
=
\frac12 g(2x\bmod1)+\operatorname{dist}(x,\Z),
\qquad x\in\mathbb T.
\]
Starting from zero, one obtains the \(n\)-th partial sum of the classical Takagi function:
\[
(\mathcal W^n0)(x)
=
\sum_{j=0}^{n-1}2^{-j}\operatorname{dist}(2^jx,\Z).
\]
This seam-touching forcing is not covered by \Cref{cor:binary-affine}, but it has the
source--collation realization described in \cite[\S7.3.1]{approx}.  Indeed, for the tent map
\(H(x):=2\operatorname{dist}(x,\Z)\), set
\[
y_0=x,\qquad u_0=0,
\qquad
y_k=H(y_{k-1}),\qquad
u_k=u_{k-1}+2^{-k}y_k.
\]
Since \(H^{\circ k}(x)=2\operatorname{dist}(2^{k-1}x,\Z)\), one has
\(u_n=(\mathcal W^n0)(x)\).  Each stage has fixed CPwL complexity, giving an exact fixed-width
realization of depth \(O(n)\).  This forward device works because the coefficients \(2^{-k}\) are
predetermined and branch-independent; in the general affine refinement problem they are replaced
by input-dependent ordered products of branch matrices, for which the backward memory mechanism is
needed.
\end{remark}

\section{Stage-dependent forcing in a fixed finite-dimensional span}
\label{sec:stage-dependent-forcing}

The stage-dependent refinement framework and its iterate formula were considered in \cite{loop}.
Here we record the improvement supplied by the residual memory controller: when the forcing terms
belong to a fixed finite-dimensional CPwL span, their accumulated contribution admits an exact
fixed-width realization of depth \(O(n)\), rather than the quadratic-depth realization obtained by
evaluating the individual summands separately.

Throughout the section, fix a nontrivial admissible offset \(s\in(0,1)\) and choose
\(\mar>0\) so that \(2\mar<\min\{s,1-s\}\).  Thus both the ordinary frame and the
\(s\)-offset frame are admissible.

\paragraph{The stage-dependent Horner recursion.}
Consider the nonstationary affine recursion
\[
\gamma_{r+1}=V\gamma_r+B_r,
\qquad r\ge0,
\]
where every \(B_r:\R\to\R^p\) is CPwL and supported in \([0,L]\).  Its \(n\)-th iterate satisfies
\[
\gamma_n
=
V^n\gamma_0+S_n,
\qquad
S_n
:=
\sum_{r=0}^{n-1}V^{\,n-1-r}B_r.
\]
This is the standard stage-dependent iterate formula recalled from \cite{loop}.  We concentrate on
the forcing contribution \(S_n\).

First work in one admissible frame, and write \(b_r:=\vect_s(B_r)\).  The cascade identity gives
\[
\vect_s(S_n)(x)
=
\sum_{k=0}^{n-1}
\mathsf T_{d_1(x)}\cdots\mathsf T_{d_k(x)}
\,b_{n-1-k}(R^kx),
\qquad x\in[0,1),
\]
where the empty matrix product is the identity.  This sum has the backward Horner form obtained by
setting \(U_{n-1}(x):=b_0(R^{n-1}x)\) and, for \(j=n-2,\ldots,0\), defining
\[
U_j(x)
:=
b_{n-1-j}(R^jx)
+
\mathsf T_{d_{j+1}(x)}U_{j+1}(x).
\]
Expanding the recursion shows that \(U_0(x)=\vect_s(S_n)(x)\).

The same backward-induction argument as in
\Cref{lem:vanishing-accumulated-state} shows that
if \(b_{n-1-i}(R^ix)=0\) for every \(i=j+1,\ldots,n-1\), then
\(U_{j+1}(x)=0\).

\subsection{Fixed-span forcing in one admissible frame}
\label{subsec:fixed-span-one-frame}

We first assume that the vectorized forcing terms belong to a fixed special-hat span in one
admissible frame.  Thus, for fixed special hats \(h_\nu\), vectors
\(v_\nu\in\R^{pL_s}\), and scalars \(a_\nu\), we suppose that
\begin{equation}
\label{eq:one-frame-stage-span}
b_r(x)
=
\sum_{\nu=1}^{N_B}
\lambda_{r,\nu}a_\nu h_\nu(x)v_\nu,
\qquad
r\ge0,
\end{equation}
where \(\supp h_\nu\subset[\mar,1-\mar]\).  For a final depth \(n\), set
\[
K_n
:=
\max_{\substack{0\le r<n\\1\le\nu\le N_B}}
|\lambda_{r,\nu}|.
\]

For each stage \(r\), define the lifted forcing readout by
\[
\mathcal B_r(Z)
:=
\sum_{\nu=1}^{N_B}
\lambda_{r,\nu}a_\nu H_{h_\nu}(Z)v_\nu.
\]
By \Cref{lem:lifted-special-hat-readout}, it satisfies
\[
\mathcal B_r(Z_j(x))
=
b_r(R^jx)
\]
for every valid memory state.

\begin{theorem}[One-frame fixed-span stage-dependent forcing]
\label{thm:one-frame-stage-dependent-forcing}
Assume that the vectorized forcing profiles satisfy
\eqref{eq:one-frame-stage-span} in one admissible frame.  Then there exist constants
\(C_0,C_1>0\), independent of \(n\) and of the coefficients \(\lambda_{r,\nu}\), such that
\[
S_n
:=
\sum_{r=0}^{n-1}V^{\,n-1-r}B_r
\in
\Ups_{C_0,C_1n}(\ReLU;1,p).
\]
The realizing networks may moreover be chosen with weights and biases bounded by
\(C_2(1+K_n)\Lambda^n\),
where \(C_2,\Lambda>0\) depend only on the fixed frame, the matrix mask, and the fixed
special-hat span.
\end{theorem}

\begin{proof}
Set
\[
B_*
:=
\sum_{\nu=1}^{N_B}
|a_\nu|\,\|h_\nu\|_{L^\infty(0,1)}\|v_\nu\|_\infty
\]
and
\[
\tau
:=
\max_{0\le r\le M-1}
\|\mathsf T_r\|_{\infty\to\infty}.
\]
Then \(\|b_r\|_{L^\infty}\le B_*K_n\), and the stage-dependent Horner recursion gives
\[
\|U_j(x)\|_\infty
\le
B_*K_n
\sum_{\ell=0}^{n-1-j}\tau^\ell.
\]
Consequently, the product-gadget scale may be chosen so that
\[
a_n
\le
C(1+K_n)\Lambda^n
\]
for constants depending only on the fixed data.

Use the same selectors and selected-action map as in
\S\ref{subsec:memory-selectors}, with this enlarged scale.  Outside the selector transition set,
the correct branch is selected exactly.  On the transition set, all later residuals lie in
\([0,\mar)\).  Since every fixed hat \(h_\nu\) vanishes there, one has
\(b_{n-1-i}(R^ix)=0\) for every later stage \(i\), independently of the coefficients
\(\lambda_{r,\nu}\).  
The vanishing-state observation above therefore gives \(U_{j+1}(x)=0\), and the identity
\(\Pi_a(\lambda,0)=0\) makes the selected matrix action exact.

Compute \(Z_n(x)=\mathcal Q^n(Z_0(x))\), initialize
\[
\widehat Z_{n-1}(x):=P(Z_n(x)),
\qquad
\widehat U_{n-1}(x)
:=
\mathcal B_0(\widehat Z_{n-1}(x)),
\]
and, for \(j=n-2,\ldots,0\), set
\[
\widehat Z_j(x)
:=
P(\widehat Z_{j+1}(x)),
\qquad
\widehat U_j(x)
:=
\mathcal B_{n-1-j}(\widehat Z_j(x))
+
\mathcal T_n(\widehat Z_j(x),\widehat U_{j+1}(x)).
\]
Exact backward replay and the preceding selected-action argument give, by backward induction,
\(\widehat Z_j(x)=Z_j(x)\) and \(\widehat U_j(x)=U_j(x)\).  Hence
\(\widehat U_0(x)=\vect_s(S_n)(x)\) for \(x\in[0,1)\), and equality at \(x=1\) follows by
continuity.

The forward and backward computations use \(O(n)\) fixed-dimensional blocks.  The readout
architecture is fixed because the number \(N_B\) of templates is independent of \(n\); only its
hardwired coefficients vary with the stage.  Thus the vectorized forcing sum has fixed width,
depth \(O(n)\), and parameter bound \(C_2(1+K_n)\Lambda^n\).  The standard finite gluing over the
translated unit intervals then yields the asserted realization of \(S_n\) on \(\R\).
\end{proof}

\subsection{General fixed-span stage-dependent affine iterates}
\label{subsec:general-fixed-span-stage-dependent}

Let \(B^{(0)},\ldots,B^{(N)}:\R\to\R^p\) be fixed CPwL templates supported in
\([0,L]\), and define
\[
B_r
:=
\sum_{\alpha=0}^{N}
\lambda_{r,\alpha}B^{(\alpha)},
\qquad
r\ge0.
\]
For a final depth \(n\), set
\[
K_n
:=
\max_{\substack{0\le r<n\\0\le\alpha\le N}}
|\lambda_{r,\alpha}|.
\]

\begin{theorem}[Fixed-span stage-dependent affine realization]
\label{thm:stage-dependent-affine-realization}
Let \(M\ge3\), let the mask be finitely supported, and assume that the associated homogeneous
operator \(V\) preserves \([0,L]\).
For the forcing terms above, define
\[
S_n
:=
\sum_{r=0}^{n-1}V^{\,n-1-r}B_r.
\]
Then there exist constants \(C_0,C_1>0\), independent of \(n\) and of the coefficients
\(\lambda_{r,\alpha}\), such that
\[
S_n
\in
\Ups_{C_0,C_1n}(\ReLU;1,p),
\qquad
n\ge1.
\]
The realizing networks may moreover be chosen with weights and biases bounded by
\(C_2(1+K_n)\Lambda^n\), where \(C_2,\Lambda>0\) depend only on the fixed templates, the
matrix mask, the admissible offset, and the support margin.

If \(\gamma:\R\to\R^p\) is CPwL and supported in \([0,L]\), and
\(W_r\gamma:=V\gamma+B_r\), 
then there exist constants \(C_0',C_1',C_2',\Lambda'>0\), independent of \(n\) and of the
coefficients \(\lambda_{r,\alpha}\), such that
\[
W_{n-1}\cdots W_0\gamma
\in
\Ups_{C_0',C_1'n}(\ReLU;1,p),
\qquad
n\ge1,
\]
and the realizing networks have weights and biases bounded by
\(C_2'(1+K_n)(\Lambda')^n\).
\end{theorem}

\begin{proof}
Apply \Cref{lem:finite-hat-decomposition} once to each fixed template.  Thus, for every
\(\alpha\), one has
\[
B^{(\alpha)}(t)
=
\sum_{\nu=1}^{N_\alpha}
a_{\alpha,\nu}
h_{\alpha,\nu}(t-\delta_{\alpha,\nu})
e_{\mu_{\alpha,\nu}},
\]
where
\(\delta_{\alpha,\nu}\in\Z\cup(\Z+s)\), and each \(h_{\alpha,\nu}\) is a special hat.
The total number of atoms is fixed independently of \(n\).

Substituting these decompositions into \(S_n\), we obtain a finite sum of terms of the form
\[
\sum_{r=0}^{n-1}
\lambda_{r,\alpha}a_{\alpha,\nu}
V^{\,n-1-r}
\bigl(h_{\alpha,\nu}(\,\cdot-\delta_{\alpha,\nu})
e_{\mu_{\alpha,\nu}}\bigr).
\]
Each atom is aligned with either the ordinary frame or the fixed admissible \(s\)-offset frame.
For a fixed atom, its stage-dependent coefficients are
\(\lambda_{r,\alpha}a_{\alpha,\nu}\), whose absolute values are bounded by a fixed multiple of
\(K_n\).  Hence \Cref{thm:one-frame-stage-dependent-forcing} applies in the corresponding frame.

Since only finitely many atoms occur, their networks can be run in parallel and summed by a final
affine layer.  This enlarges the width by only a fixed factor, preserves depth \(O(n)\), and gives
the parameter bound \(C_2(1+K_n)\Lambda^n\).  This proves the assertion for \(S_n\).

For the affine iterates, the standard stage-dependent identity recalled from \cite{loop} gives
\[
W_{n-1}\cdots W_0\gamma
=
V^n\gamma+S_n.
\]
The homogeneous term is realized by
\Cref{thm:imported-homogeneous}.  Running the homogeneous and forcing networks in parallel and
adding their outputs proves the final assertion, including the stated parameter bound after
enlarging the constants.
\end{proof}

\begin{remark}
For \(M=2\), the same argument applies when every fixed template is ordinary-frame
seam-separated in the sense of \Cref{def:binary-seam-separated}.
Since the family of templates is finite, their individual seam separations may be replaced by a
common positive margin.
\end{remark}

\section{Reductions and geometric consequences}
\label{sec:reductions-geometric}

The recursive geometric constructions developed in \cite{loop} include open curves, finitely many
curve states, and stage-dependent connector data.  Their reduction to vector-valued refinement is
algebraic and does not depend on the particular realization mechanism.  We record only the parts
needed to apply the affine and fixed-span theorems proved above.

\subsection{Anchored defects and finite-state systems}

An \emph{anchored profile} is a CPwL map
\(\Gamma:\R\to\R^p\) that is constant outside a bounded interval.  Subtracting such a profile turns
an open curve with fixed tails into a compactly supported defect.

More generally, consider a stage-dependent recursion
\[
\gamma_{r+1}
=
V\gamma_r+B_r
\]
and write \(\gamma_r=\Gamma_r+\eta_r\), where \(\Gamma_r\) is a prescribed anchored profile.
Then the defect satisfies
\begin{equation}
\label{eq:anchored-defect-recursion}
\eta_{r+1}
=
V\eta_r+E_r,
\qquad
E_r
:=
V\Gamma_r+B_r-\Gamma_{r+1}.
\end{equation}
Indeed, this follows immediately by substituting
\(\gamma_r=\Gamma_r+\eta_r\) into the recursion.

The compact support of \(E_r\) is determined entirely by the tails.  To make this explicit, set
\(A_*:=\sum_{j\in\Z}A_j\).  If \(\Gamma_r\), \(\Gamma_{r+1}\), and \(B_r\) have constant
left and right tails denoted by \(\Gamma_{r,\pm}\), \(\Gamma_{r+1,\pm}\), and \(B_{r,\pm}\),
respectively, then \(E_r\) is compactly supported whenever
\[
A_*\Gamma_{r,\pm}+B_{r,\pm}
=
\Gamma_{r+1,\pm}.
\]
For a stationary anchor \(\Gamma_r=\Gamma\) and stationary forcing \(B_r=B\), this reduces to the
usual compatibility condition for \(E=W\Gamma-\Gamma\).

\begin{corollary}[Anchored-profile reduction]
\label{cor:anchored-profile-reduction}
Assume that \(V\) preserves a support window \([0,L]\), that \(\eta_0\) is CPwL and supported in
\([0,L]\), and that the defect forcing in
\eqref{eq:anchored-defect-recursion} has the fixed-span form
\[
E_r
=
\sum_{\alpha=0}^{N}\lambda_{r,\alpha}E^{(\alpha)}
\]
for fixed CPwL templates \(E^{(\alpha)}\) supported in \([0,L]\).
If \(M\ge3\), then every
\(\eta_n\) admits an exact fixed-width ReLU realization of depth \(O(n)\).  The same conclusion
holds for \(M=2\) when the templates \(E^{(\alpha)}\) are ordinary-frame seam-separated.

If, in addition, the profiles \(\Gamma_n\) belong to a fixed finite-dimensional CPwL span, then the
same conclusion holds for \(\gamma_n=\Gamma_n+\eta_n\).
\end{corollary}

\begin{proof}
The defect recursion is a stage-dependent affine refinement system, so
\Cref{thm:stage-dependent-affine-realization} applies to \(\eta_n\).  A fixed finite-dimensional
CPwL realization of \(\Gamma_n\) can then be added in parallel without changing the asymptotic
width or depth.
\end{proof}

Finite-state recursive systems require no additional realization argument.  Suppose that
\(\gamma=(\gamma_1,\ldots,\gamma_R)\) satisfies
\[
(\mathfrak W\gamma)_a(t)
=
\sum_{b=1}^{R}\sum_{j\in\Z}
A_j^{ab}\gamma_b(Mt-j)+B_a(t),
\qquad
a=1,\ldots,R.
\]
Stacking the state components into the map
\[
G(t):=\bigl(\gamma_1(t),\ldots,\gamma_R(t)\bigr)\in\R^{pR}
\]
converts this system
into an ordinary vector-valued
affine refinement operator whose \(j\)-th mask matrix has \((a,b)\)-block \(A_j^{ab}\).
Consequently, the main affine theorem and its stage-dependent extension apply directly to the
stacked system.  Any individual state is recovered by a fixed coordinate projection.

\subsection{Recursive curve generators}

Copy-and-connector constructions provide a common geometric source of the defect recursion
\eqref{eq:anchored-defect-recursion}.  In such a construction, the scaled and transformed copies of
the preceding curve determine the fixed homogeneous operator \(V\), while connectors and changes
of anchor contribute to \(E_r\).  Detailed versions of this reduction, including the corresponding
finite-state formulations, are given in \cite{loop}.

The new realization theorem applies once the resulting defect forcing lies in a fixed
finite-dimensional CPwL span.  Notice that this condition concerns \(E_r\) itself.  Since
\[
E_r
=
V\Gamma_r+B_r-\Gamma_{r+1},
\]
its coefficients may depend on the parameters describing both stages \(r\) and \(r+1\); they need
not coincide with the coefficients used to represent \(\Gamma_r\) or the connector data
separately.

\begin{corollary}[Linear-depth upgrade for geometric recursions]
\label{cor:geometric-linear-depth-upgrade}
Assume that \(V\) preserves a support window \([0,L]\), and consider a recursive curve
construction whose anchored and, if necessary, stacked defect satisfies
\[
\eta_{r+1}
=
V\eta_r+
\sum_{\alpha=0}^{N}\lambda_{r,\alpha}E^{(\alpha)},
\]
where \(\eta_0\) and the fixed CPwL templates \(E^{(\alpha)}\) are supported in \([0,L]\).
If \(M\ge3\), then the \(n\)-th defect admits an exact fixed-width ReLU realization of depth
\(O(n)\).  The same conclusion holds for \(M=2\) if every template \(E^{(\alpha)}\) is
ordinary-frame seam-separated.

If
\[
K_n
:=
\max_{\substack{0\le r<n\\0\le\alpha\le N}}
|\lambda_{r,\alpha}|,
\]
the weights and biases may be bounded by
\(C(1+K_n)\Lambda^n\), with constants depending only on the fixed refinement and template data.
\end{corollary}

\begin{proof}
This is an immediate application of
\Cref{thm:stage-dependent-affine-realization}, followed, when needed, by the fixed anchored-profile
and coordinate-projection readouts described above.
\end{proof}

The Hilbert- and Morton-type constructions described in \cite{loop} fit this criterion after their
anchored defects are formed: the stage dependence is carried by finitely many fixed anchor and
connector profiles, with coefficients governed by the corresponding endpoint scales.  The
corollary therefore upgrades the previously available realization of these stage-dependent affine
recursions to fixed width and depth \(O(n)\).  We do not repeat their explicit copy matrices,
translations, and connector formulas here.

\section{Conclusions}
\label{sec:conclusions}

We have proved exact fixed-width, depth-\(O(n)\) ReLU realization results for affine
one-dimensional refinement iterates with vector-valued CPwL data.  The homogeneous contribution is
supplied by the loop-controller theorem of \cite{loop}; the new contribution is a linear-depth
realization of the affine forcing sum.  For every \(M\ge3\), arbitrary compactly supported CPwL
forcing terms are covered by combining the ordinary frame with a nontrivial admissible offset
frame.  The resulting networks have weights and biases growing at most exponentially in \(n\),
according to the coarse estimates used here.

The central new mechanism is the residual memory controller.  It replaces the noninvertible
residual dynamics by an injective skew-product and thereby permits exact backward replay of the
residual states required by the Horner recursion.  Complementary loop readouts then recover the
forcing values exactly on those states.  The remaining branch-selector ambiguity is handled by a
finite-horizon trapping argument: whenever a residual enters a selector transition strip, all later
forcing samples vanish, and hence the accumulated affine state being multiplied is zero.  Thus the
memory controller solves the chronology problem, while the vanishing-state mechanism resolves the
remaining selector ambiguity.

The construction also applies to stage-dependent forcing terms lying in a fixed
finite-dimensional CPwL span.  In particular, the anchored-profile, finite-state, and
copy-and-connector reductions developed in \cite{loop} inherit the improved linear-depth bound
whenever their defect forcing belongs to such a span.  This includes the Hilbert- and Morton-type
recursive constructions discussed there.  Open curves are reduced to compactly supported defects
by subtracting anchored profiles, while finite-state systems are reduced to ordinary vector-valued
refinement by stacking their components.

The present results concern exact realization of finite iterates rather than convergence to limiting
fractal objects.  They are also one-dimensional in the parameter and CPwL in regularity.  For
\(M=2\), no nontrivial admissible offset exists, and the current argument therefore covers only
ordinary-frame seam-separated forcing terms; arbitrary binary CPwL forcing requires a more general
seam-exact mechanism.  Extending the memory construction to higher-dimensional residual dynamics
and sharpening the present parameter bounds remain natural directions for further work.

\section*{Acknowledgments}

This work was supported by the Natural Sciences and Engineering Research Council of Canada
through its Discovery Grants program.

\end{document}